\documentclass[letterpaper]{article}
\usepackage[utf8]{inputenc}
\usepackage{amsmath}
\usepackage{amsthm}
\usepackage{amssymb}
\usepackage{graphicx}
\usepackage[margin=1in,letterpaper]{geometry}
\usepackage{enumitem}
\usepackage{xcolor}
\usepackage[colorlinks=true,allcolors=blue]{hyperref}
\usepackage{subcaption}
\usepackage{algorithm}
\usepackage[authoryear]{natbib}
\usepackage{comment}

\newtheorem{theorem}{Theorem}
\newtheorem{corollary}[theorem]{Corollary}
\newtheorem{lemma}[theorem]{Lemma}
\newtheorem*{remark}{Remark}

\DeclareMathOperator*{\argmin}{arg\,min}

\begin{document}

\title{Classification with Nearest Disjoint Centroids}
\author{Nicolas Fraiman \\  \href{mailto:fraiman@email.unc.edu}{fraiman@email.unc.edu} 
\and Zichao Li \\ 
\href{mailto:lizichao@live.unc.edu}{lizichao@live.unc.edu}}
\date{}
\maketitle

\begin{abstract}
In this paper, we develop a new classification method based on nearest centroid, and it is called the nearest disjoint centroid classifier. Our method differs from the nearest centroid classifier in the following two aspects: (1) the centroids are defined based on disjoint subsets of features instead of all the features, and (2) the distance is induced by the dimensionality-normalized norm instead of the Euclidean norm. We provide a few theoretical results regarding our method. In addition, we propose a simple algorithm based on adapted $k$-means clustering that can find the disjoint subsets of features used in our method, and extend the algorithm to perform feature selection. We evaluate and compare the performance of our method to other classification methods on both simulated data and real-world gene expression datasets. The results demonstrate that our method is able to outperform other competing classifiers by having smaller misclassification rates and/or using fewer features in various settings and situations.
\end{abstract}

\section{Introduction}\label{sec:intro}
In general, classification is the task of predicting the category that an observation belongs to. Many applications in real life are dealing with classification problems, such as determining the category of images \citep{russakovsky2015imagenet}, deciding the topic of documents \citep{lewis2004rcv1}, and diagnosing the cancer type of tissues \citep{edgar2002gene}. In order to perform classification, people often build models that can automatically make predictions by identifying patterns in a dataset that includes observations and their associated class labels. Over the years, researchers have developed many different classification methods, ranging from simpler ones such as logistic regression \citep{hastie2009elements}, $k$-nearest neighbors \citep{cover1967nearest}, Naive Bayes \citep{hand2001idiot}, and decision trees \citep{breiman1984classification}, to more complicated ones such as support vector machines \citep{cortes1995support}, random forest \citep{breiman2001random}, gradient boosting machine \citep{friedman2001greedy}, and neural networks \citep{krizhevsky2012imagenet}.

Among numerous existing classification methods, the nearest centroid classifier is one of the simplest classification method. It computes the centroid of each class as the average of the training samples that belong to that class, and classifies a test sample to the class with the nearest centroid. Intuitively, the nearest centroid classifier can be understood as creating one prototype to represent each class, and making predictions by selecting the class with the prototype that is most similar to the test sample. The theoretical simplicity makes it easy to understand and interpret, and the computational efficiency makes it appealing in practice. Therefore, it has been used in many different fields of application, including gene expression analysis \citep{tibshirani2002diagnosis, levner2005feature, dabney2005classification, dabney2007optimality} and text classification \citep{han2000centroid, lertnattee2004effect, tan2008improved}.

The nearest shrunken centroid classifier, a simple modification of the nearest centroid classifier, was proposed by \citet{tibshirani2002diagnosis}. It works by shrinking the class centroids toward the overall centroids after standardizing each feature by the pooled within-class standard deviation of that feature. Importantly, the shrinkage process can be considered as performing feature selection, which is desirable in applications with high-dimensional features. For example, when predicting the cancer type using a gene expression dataset, the nearest shrunken centroid classifier would select a small subset of the genes to make predictions, whereas the nearest centroid classifier would use all the genes. This characteristic makes the nearest shrunken centroid classifier a popular method in gene expression analysis \citep{tibshirani2003class, sorlie2003repeated, volinia2006microrna, parker2009supervised, curtis2012genomic}.

In this paper, we develop a new classification method based on nearest centroid, and it is called the nearest disjoint centroid classifier. The main idea is to define the centroids based on disjoint subsets of features instead of all the features. More specifically, we partition the features into $k$ groups, each one corresponding to one of the $k$ classes, and the centroid for each class is defined using the corresponding group of features. In order to find the $k$ disjoint subsets of features, we propose a simple algorithm based on adapted $k$-means clustering. A similar formulation was proposed in \citet{fraiman2020biclustering}, in which the authors used an alternating $k$-means algorithm with disjoint centroids to perform biclustering. However, our method applies to supervised classification problems rather than unsupervised biclustering problems. Importantly, this means that we assume there are $k$ disjoint subsets of features with discriminative power, which is generally true for high-dimensional data where the number of features $p$ is much larger than the number of classes $k$. In addition, our method is able to perform feature selection by adding a special cluster that represents a ``global'' baseline, and features assigned to the special cluster are not used in making predictions.

The rest of this paper is organized as follows. In Section \ref{sec:problem}, we formulate the problem and give a high-level description of our nearest disjoint centroid classifier. In Section \ref{sec:theory}, we present and prove a few theoretical results regarding our method. In Section \ref{sec:proof}, we provide a rigorous proof of the main consistency result. In Section \ref{sec:algorithm}, we propose a simple algorithm based on adapted $k$-means clustering that finds the disjoint subsets of features used in our method. In Section \ref{sec:selection}, we extend our method to perform feature selection by assigning features to a special cluster that is not used for classification. In Section \ref{sec:simulation}, we evaluate and compare the performance of our nearest disjoint centroid classifier on simulated data to other classification methods. In Section \ref{sec:applications}, we apply our method to three cancer gene expression datasets, and show that our method is able to outperform other competing classifiers by having smaller misclassification rates and/or using fewer features. In Section \ref{sec:discussion}, we conclude with a discussion.

\section{Problem Formulation}\label{sec:problem}
Suppose we are given $n$ training samples and their associated class labels $(X_1, Y_1), \ldots, (X_n, Y_n)$ where $X_i \in \mathbb{R}^p$ and $Y_i \in \{1, \ldots, k\}$ for $1 \le i \le n$. For $1 \le j \le k$, let $S_j$ denote the set of indices of training samples that belong to class $j$. Throughout this paper, we assume $k \le \min(n, p)$, and all $S_j$ are non-empty. The nearest centroid classifier works by first computing the per-class centroid $c_j$ as 
\begin{equation*}
    c_j = \frac{1}{|S_j|} \sum_{i \in S_j} X_i,\ 1 \le j \le k.
\end{equation*}
Then, it classifies a test sample $X$ to the class $Y$ with the nearest centroid. When using Euclidean distance, we have
\begin{equation*}
    Y = \argmin_{1 \le j \le k} ||X - c_j||_2^2.
\end{equation*}
Essentially, the centroids $c_j$ are all vectors in $\mathbb{R}^p$, and they are computed by minimizing the following objective function:
\begin{equation*}
    \sum_{j=1}^{k} \sum_{i \in S_j} ||X_i - c_j||_2^2.
\end{equation*}

Now, suppose the centroids are defined using disjoint subsets of features instead of all the features. More specifically, let $I = \{1, \ldots, p\}$ be the index set of features, then $I$ could be partitioned into $k$ disjoint nonempty sets $I_1, \ldots, I_k$, where $I_1 \cup \cdots \cup I_k = I$. For any $X = (x_1, \ldots, x_p) \in \mathbb{R}^p$, let $X(I_j) = (x_i)_{i \in I_j}$. The space of $X(I_j)$ is defined as $\mathbb{R}^{I_j}$, and we define the dimensionality-normalized norm on $\mathbb{R}^{I_j}$ as 
\begin{equation*}
    ||X(I_j)||_{dn} = \sqrt{\frac{\sum_{i \in I_j} x_i^2}{l_j}},
\end{equation*}
where $l_j = |I_j|$ denote the cardinality of the index set $I_j$, and it is also the dimension of the space $\mathbb{R}^{I_j}$.

In our method, we would like to find the $k$ disjoint subsets of features $I_1, \ldots, I_k$ and the corresponding $k$ disjoint centroids $c_j \in \mathbb{R}^{I_j}, 1 \le j \le k$ such that the following objective function is minimized:
\begin{equation}\label{ndc_objective}
    \sum_{j=1}^{k} \sum_{i \in S_j} ||X_i(I_j) - c_j||_{dn}^2.
\end{equation}
For a test sample $X$, we would classify it to the class $Y$ with the nearest disjoint centroid (distance induced by the dimensionality-normalized norm), which is given by
\begin{equation*}
    Y = \argmin_{1 \le j \le k} ||X(I_j) - c_j||_{dn}^2.
\end{equation*}

The reason of using the dimensionality-normalized norm instead of the Euclidean norm in the objective function (\ref{ndc_objective}) is twofold:
\begin{enumerate}
    \item Theoretically, it is important for each individual term $||X(I_j) - c_j||_{dn}^2$ to be appropriately normalized, so that the objective function (\ref{ndc_objective}) is summing up $n$ roughly comparable terms no matter how large or small each subset of features $I_j$ is. If we use the Euclidean norm in the objective function, then when the data is imbalanced (some classes have much more observations that other classes), the majority classes would be assigned much smaller subsets of features, and the minority classes would be assigned much larger subsets of features. This is because such assignment would minimize the objective function by minimizing the inner sum $\sum_{i \in S_j} ||X_i(I_j) - c_j||_{2}^2$ for each class $j$.
    \item Empirically, we conducted several simulations to compare the performance of our method using the dimensionality-normalized norm and the Euclidean norm. Although they are not included in the paper, the results confirmed our expectation that using Euclidean norm would lead to worse performance when the data is imbalanced.
\end{enumerate}

It is easy to see that if the $k$ disjoint subsets of features $I_1, \ldots, I_k$ are given, then the corresponding $k$ centroids can be computed as
\begin{equation*}
    c_j = \frac{1}{|S_j|} \sum_{i \in S_j} X_i(I_j),\ 1 \le j \le k.
\end{equation*}
However, finding the best disjoint subsets of features $I_1, \ldots, I_k$ to minimize the objective function (\ref{ndc_objective}) is a combinatorial optimization problem, which is computationally intractable for large $p$. In light of this fact, we present a simple algorithm in Section \ref{sec:algorithm} that uses adapted $k$-means clustering to find the disjoint subsets of features $I_1, \ldots, I_k$.

Again, we emphasize that our method assumes that there are $k$ disjoint subsets of features with discriminative power, which might not apply to all data, but it is generally true for high-dimensional data where the number of features $p$ is much larger than the number of classes $k$. In addition, it is possible to extend our method to handle the more general case by allowing the centroids to have a common set of features, which could be selected by running any feature selection algorithm. In that case, our main theoretical result Theorem \ref{thm:consistency} would still hold, and our main algorithms Algorithm \ref{algorithm1} and Algorithm \ref{algorithm2} would only need to be slightly modified.

\section{Theoretical Results}\label{sec:theory}
Suppose the training samples and their associated class labels $(X_1, Y_1), \ldots, (X_n, Y_n)$ are i.i.d.\ with the same distribution as $(X, Y)$ where $X \in \mathbb{R}^p$ and $Y \in \{1, \ldots, k\}$. Let $\mu$ denote the distribution of $(X, Y)$, and let $\mu_n$ denote the empirical distribution of the $n$ training samples and their associated class labels. In addition, suppose that for $1 \le j \le k$, the probability of $Y=j$ is given by $p^{(j)}$:
\begin{equation*}
    P(Y = j) = p^{(j)},
\end{equation*}
and the conditional distribution of $X$ given $Y=j$ is given by $\mu^{(j)}$:
\begin{equation*}
    X | Y=j \sim \mu^{(j)}.
\end{equation*}
Similarly, for $1 \le j \le k$, let $p_n^{(j)}$ denote the empirical proportion of $Y_i=j$, and let $\mu_n^{(j)}$ denote the empirical conditional distribution of $X_i$ given $Y_i=j$ in the training data.

We minimize the empirical risk defined as
\begin{align*}
    W(\mathbf{I}, \mathbf{c}, \mu_n) &= \frac{1}{n} \sum_{i=1}^{n} \sum_{j=1}^{k} \mathbf{1}_{\{Y_i = j\}} \cdot ||X_i(I_j) - c_j||_{dn}^2 \nonumber \\
    &= \sum_{j=1}^{k} p_n^{(j)} \int ||x(I_j) - c_j||_{dn}^2 d\mu_n^{(j)}(x)
\end{align*}
over all feature subsets $\mathbf{I} = \{I_1, \ldots, I_k\}$ and centroids $\mathbf{c} = \{c_1, \ldots, c_k\}$. The risk is defined as
\begin{align*}
    W(\mathbf{I}, \mathbf{c}, \mu) &= \int \sum_{j=1}^{k} \mathbf{1}_{\{y = j\}} \cdot ||x(I_j) - c_j||_{dn}^2 d\mu(x,y) \nonumber \\
    &= \sum_{j=1}^{k} p^{(j)} \int ||x(I_j) - c_j||_{dn}^2 d\mu^{(j)}(x).
\end{align*}
The optimal risk is defined as
\begin{equation*}
    W^*(\mu) = \inf_{\mathbf{I}} \inf_{\mathbf{c}}  W(\mathbf{I}, \mathbf{c}, \mu).
\end{equation*}

For a fixed feature subset $I_j$, it is easy to verify that 
\begin{equation*}
    \argmin_{c_j} \int ||x(I_j) - c_j||_{dn}^2 d\mu_n^{(j)}(x) = \int x(I_j) d\mu_n^{(j)}(x),
\end{equation*}
and 
\begin{equation*}
    \argmin_{c_j} \int ||x(I_j) - c_j||_{dn}^2 d\mu^{(j)}(x) = \int x(I_j) d\mu^{(j)}(x).
\end{equation*}

Let $\delta_n \ge 0$. A feature subsets $\mathbf{I}_n$ and centroids $\mathbf{c}_n$ as a whole is a $\delta_n$-minimizer of the empirical risk if 
\begin{equation*}
    W(\mathbf{I}_n, \mathbf{c}_n, \mu_n) \le W^*(\mu_n) + \delta_n,
\end{equation*}
where $W^*(\mu_n) = \inf_\mathbf{I} \inf_\mathbf{c} W(\mathbf{I}, \mathbf{c}, \mu_n)$. When $\delta_n = 0$, $\mathbf{I}_n$ and $\mathbf{c}_n$ as a whole is called an empirical risk minimizer. Since $\mu_n$ is supported on at most $n$ points, the existence of an empirical risk minimizer is guaranteed.

The first theoretical result of this paper is the following consistency theorem, which states that the risk of a $\delta_n$-minimizer of the empirical risk converges to the optimal risk as long as $\lim_{n \to \infty} \delta_n = 0$.

\begin{theorem}\label{thm:consistency}
Assume that all $\mu^{(j)}$ have finite second moments that are bounded by a constant $h$:
\begin{equation*}
    \max_{1 \le j \le k} \int ||x||_2^2 d\mu^{(j)}(x) \le h.
\end{equation*}
Let $\mathbf{I}_n$ and $\mathbf{c}_n$ be a $\delta_n$-minimizer of the empirical risk. If $\lim_{n \to \infty} \delta_n = 0$, then 
\begin{equation*}
    \lim_{n \to \infty} W(\mathbf{I}_n, \mathbf{c}_n, \mu) = W^*(\mu)\ a.s.
\end{equation*}
\end{theorem}

A detailed proof of Theorem \ref{thm:consistency} is given in Section \ref{sec:proof}, which also includes a remark at the end that provides some analysis on the rate of convergence.

We can characterize the feature subsets $\mathbf{I}^* = \{I_1^*, \ldots, I_k^*\}$ and the centroids $\mathbf{c}^* = \{c_1^*, \ldots, c_k^*\}$ that achieve the optimal risk $W^*(\mu)$ if we make some additional assumptions.
\begin{theorem}\label{thm:minimizer}
Assume that the feature subsets $\mathbf{I}^* = \{I_1^*, \ldots, I_k^*\}$ is a partition of the $p$ features $I = \{1, \ldots, p\}$, and for $1 \le j \le k$ we have
\begin{equation*}
    \min_{I_j \subset I} \int || x(I_j) - c_j||_{dn}^2 d\mu^{(j)}(x) = \int || x(I_j^*) - c_j^*||_{dn}^2 d\mu^{(j)}(x),
\end{equation*}
where $c_j = \int x(I_j) d\mu^{(j)}(x)$ and $c_j^* = \int x(I_j^*) d\mu^{(j)}(x)$. Then
\begin{equation*}
    W(\mathbf{I}^*, \mathbf{c}^*, \mu) = W^*(\mu).
\end{equation*}
\end{theorem}

\begin{proof}
We prove by contradiction. Assume that there exist feature subsets $\mathbf{I}^\dagger = \{I_1^\dagger, \ldots, I_k^\dagger\}$ and centroids $\mathbf{c}^\dagger = \{c_1^\dagger, \ldots, c_k^\dagger\}$ such that 
\begin{equation*}
    W(\mathbf{I}^\dagger, \mathbf{c}^\dagger, \mu) < W(\mathbf{I}^*, \mathbf{c}^*, \mu).
\end{equation*}
Since by definition 
\begin{equation*}
     W(\mathbf{I}, \mathbf{c}, \mu) = \sum_{j=1}^{k} p^{(j)} \int ||x(I_j) - c_j||_{dn}^2 d\mu^{(j)}(x),
\end{equation*}
there must exist at least one $t \in \{1, \ldots, k\}$ such that
\begin{equation*}
    \int ||x(I_t^\dagger) - c_t^\dagger||_{dn}^2 d\mu^{(t)}(x) < \int ||x(I_t^*) - c_t^*||_{dn}^2 d\mu^{(t)}(x).
\end{equation*}
However, the property of $\mathbf{I}^*$ guarantees that
\begin{align*}
    \int || x(I_t^*) - c_t^*||_{dn}^2 d\mu^{(t)}(x) &= \min_{I_t \subset I} \int || x(I_t) - c_t||_{dn}^2 d\mu^{(t)}(x) \\
    &\le \int ||x(I_t^\dagger) - \int x(I_t^\dagger) d\mu^{(t)}(x)||_{dn}^2 d\mu^{(t)}(x) \\
    &\le \int ||x(I_t^\dagger) - c_t^\dagger||_{dn}^2 d\mu^{(t)}(x).
\end{align*}
where $c_t^* = \int x(I_t^*) d\mu^{(t)}(x)$ and $c_t = \int x(I_t) d\mu^{(t)}(x)$. Now we have a contradiction, and therefore such $\mathbf{I}^\dagger$ and $\mathbf{c}^\dagger$ could not exist.
\end{proof}

\begin{corollary}\label{cor:minimizer}
Assume that the $p$ features are all independent and consist of $k$ successive blocks, each of size $d$. In addition, assume that for $1\le j \le k$ and $X = (x_1, \ldots, x_p) \sim \mu^{(j)}$, the $d$ entries in the $j$-th block $x_{(j-1)d+1}, \ldots, x_{jd}$ are i.i.d.\ with variance $\sigma_1^2$, and the rest of the $p-d$ entries are i.i.d.\ with variance $\sigma_2^2$, with $\sigma_1^2 < \sigma_2^2$. For $1\le j \le k$, if we let 
\begin{equation*}
    I_j^* = \{x_{(j-1)d+1}, \ldots, x_{jd}\}, 
\end{equation*}
and $c_j^* = \int x(I_j^*) d\mu^{(j)}(x)$, then $W(\mathbf{I}^*, \mathbf{c}^*, \mu) = W^*(\mu)$. 
\end{corollary}

\begin{proof}
This is a direct corollary of Theorem \ref{thm:minimizer}. We only need to verify that for $1 \le j \le k$ we have
\begin{equation*}
    \min_{I_j \subset I} \int || x(I_j) - c_j||_{dn}^2 d\mu^{(j)}(x) = \int || x(I_j^*) - c_j^*||_{dn}^2 d\mu^{(j)}(x),
\end{equation*}
where $c_j = \int x(I_j) d\mu^{(j)}(x)$ and $c_j^* = \int x(I_j^*) d\mu^{(j)}(x)$. By assumption, for any specific feature $I_j = \{i\}$, 
\begin{equation*}
    \int || x(I_j) - c_j||_2^2 d\mu^{(j)}(x) = \sigma_1^2
\end{equation*}
if $i \in I_j^*$, and 
\begin{equation*}
    \int || x(I_j) - c_j||_2^2 d\mu^{(j)}(x) = \sigma_2^2 > \sigma_1^2
\end{equation*}
if $i \notin I_j^*$. This means that for any feature subset $I_j \subset I_j^*$, we have
\begin{equation*}
    \int || x(I_j) - c_j||_{dn}^2 d\mu^{(j)}(x) = \sigma_1^2.
\end{equation*}
In addition, for any other feature subset $I_j$ that includes at least one feature $i \notin I_j^*$, we have
\begin{equation*}
    \int || x(I_j) - c_j||_{dn}^2 d\mu^{(j)}(x) > \sigma_1^2.
\end{equation*}
Hence for $1 \le j \le k$ we have
\begin{equation*}
    \min_{I_j \subset I} \int || x(I_j) - c_j||_{dn}^2 d\mu^{(j)}(x) = \sigma_1^2 = \int || x(I_j^*) - c_j^*||_{dn}^2 d\mu^{(j)}(x),
\end{equation*}
and the proof is completed.
\end{proof}

\section{Proof of the Main Theoretical Result}\label{sec:proof}
In this section, we give a detailed proof of Theorem \ref{thm:consistency}, which is our main theoretical result. Recall that the $L_2$ Wasserstein distance between two probability measures $\mu_1$ and $\mu_2$ on $\mathbb{R}^p$, with finite second moment, is defined as
\begin{equation*}
    \gamma(\mu_1, \mu_2) = \inf_{X_1 \sim \mu_1, X_2 \sim \mu_2} (\mathbb{E}||X_1 - X_2||_2^2)^{1/2},
\end{equation*}
where the infimum is taken over all joint distributions of two random variables $X_1$ and $X_2$ such that $X_1$ has distribution $\mu_1$ and $X_2$ has distribution $\mu_2$. It has been shown in \citet{rachev1998mass} that $\gamma$ is a metric on the space of probability distributions on $\mathbb{R}^p$ with finite second moment, and that the infimum is a minimum and can be achieved.

We first prove the following four lemmas.

\begin{lemma}\label{lem:riskbound1}
For any feature subset $I_j$ and centroid $c_j$, we have
\begin{equation*}
    \left| \left[ \int ||x(I_j) - c_j||_{dn}^2 d\mu_1(x) \right]^{1/2} - \left[ \int ||x(I_j) - c_j||_{dn}^2 d\mu_2(x) \right]^{1/2} \right| \le \gamma(\mu_1, \mu_2).
\end{equation*}
\end{lemma}

\begin{proof}
Let $X_1 \sim \mu_1$ and $X_2 \sim \mu_2$ achieve the infimum defining $\gamma(\mu_1, \mu_2)$. Then
\begin{equation*}
    \left[ \int ||x(I_j) - c_j||_{dn}^2 d\mu_1(x) \right]^{1/2} 
    = \left[ \mathbb{E} \frac{||X_1(I_j) - c_j||_2^2}{l_j} \right]^{1/2} 
    \le \left[ \mathbb{E} \frac{(||X_1(I_j) - X_2(I_j)||_2 + ||X_2(I_j) - c_j||_2)^2}{l_j} \right]^{1/2}.
\end{equation*}
Using Cauchy–Schwarz inequality, we have
\begin{align*}
    &\mathbb{E} \left[ \frac{(||X_1(I_j) - X_2(I_j)||_2 + ||X_2(I_j) - c_j||_2)^2}{l_j} \right] \\
    &\le \mathbb{E} ||X_1-X_2||_2^2 + \mathbb{E} \left[ \frac{||X_2(I_j) - c_j||_2^2}{l_j} \right] + 2 \mathbb{E} \left[ ||X_1 - X_2||_2 \cdot \frac{||X_2(I_j) - c_j||_2}{\sqrt{l_j}} \right] \\
    &\le \mathbb{E} ||X_1-X_2||_2^2 + \mathbb{E} \left[ \frac{||X_2(I_j) - c_j||_2^2}{l_j} \right] + 2 \left[ \mathbb{E} ||X_1 - X_2||_2^2 \right]^{1/2} \left[ \mathbb{E} \frac{||X_2(I_j) - c_j||_2^2}{l_j} \right]^{1/2} \\
    &= \left( \left[ \mathbb{E} ||X_1 - X_2||_2^2 \right]^{1/2} + \left[ \mathbb{E} \frac{||X_2(I_j) - c_j||_2^2}{l_j} \right]^{1/2} \right)^2.
\end{align*}
Consequently
\begin{align*}
    \left[ \int ||x(I_j) - c_j||_{dn}^2 d\mu_1(x) \right]^{1/2} 
    &\le \left[ \mathbb{E} ||X_1 - X_2||_2^2 \right]^{1/2} + \left[ \mathbb{E} \frac{||X_2(I_j) - c_j||_2^2}{l_j} \right]^{1/2} \\
    &= \gamma(\mu_1, \mu_2) + \left[ \int ||x(I_j) - c_j||_{dn}^2 d\mu_2(x) \right]^{1/2},
\end{align*}
which implies that 
\begin{equation*}
    \left[ \int ||x(I_j) - c_j||_{dn}^2 d\mu_1(x) \right]^{1/2} - \left[ \int ||x(I_j) - c_j||_{dn}^2 d\mu_2(x) \right]^{1/2} \le \gamma(\mu_1, \mu_2).
\end{equation*}
The other direction can be proved similarly.
\end{proof}

\begin{lemma}\label{lem:riskbound2}
For any feature subset $I_j$ and centroid $c_j$, if
\begin{equation*}
    \max \left( \int ||x(I_j) - c_j||_{dn}^2 d\mu_1(x), \int ||x(I_j) - c_j||_{dn}^2 d\mu_2(x) \right) \le h,
\end{equation*}
then
\begin{equation*}
    \left| \int ||x(I_j) - c_j||_{dn}^2 d\mu_1(x) - \int ||x(I_j) - c_j||_{dn}^2 d\mu_2(x) \right| \le 2\sqrt{h} \gamma(\mu_1, \mu_2).
\end{equation*}
\end{lemma}

\begin{proof}
Let $a = \left[ \int ||x(I_j) - c_j||_{dn}^2 d\mu_1(x) \right]^{1/2}$ and $b = \left[ \int ||x(I_j) - c_j||_{dn}^2 d\mu_2(x) \right]^{1/2}$. Then
\begin{equation*}
    \left| \int ||x(I_j) - c_j||_{dn}^2 d\mu_1(x) - \int ||x(I_j) - c_j||_{dn}^2 d\mu_2(x) \right| = |a^2 - b^2| = |a+b||a-b| \le 2\sqrt{h}\gamma(\mu_1, \mu_2),
\end{equation*}
where the last inequality follows from Lemma \ref{lem:riskbound1}.
\end{proof}

\begin{lemma}\label{lem:riskbound3}
Recall that $\mu$ denote the distribution of $(X, Y)$, and is associated with $p^{(j)}$ and $\mu^{(j)}$ for $1 \le j \le k$. In addition, $\mu_n$ denote the empirical distribution of $(X_1, Y_1), \ldots, (X_n, Y_n)$, and is associated with $p_n^{(j)}$ and $\mu_n^{(j)}$ for $1 \le j \le k$. For any feature subsets $\mathbf{I}$ and centroids $\mathbf{c}$, if
\begin{equation*}
    \max_{I_j \in \mathbf{I}, c_j \in \mathbf{c}} \left( \int ||x(I_j) - c_j||_{dn}^2 d\mu^{(j)}(x), \int ||x(I_j) - c_j||_{dn}^2 d\mu_n^{(j)}(x) \right) \le h,
\end{equation*}
then
\begin{equation*}
    \left| W(\mathbf{I}, \mathbf{c}, \mu) - W(\mathbf{I}, \mathbf{c}, \mu_n) \right| \le 2k\sqrt{h} \max_{1 \le j \le k} \gamma(\mu^{(j)}, \mu_n^{(j)}) + kh \max_{1 \le j \le k} \left| p^{(j)} - p_n^{(j)} \right|.
\end{equation*}
\end{lemma}

\begin{proof}
By triangle inequality, we have
\begin{align}
    \left| W(\mathbf{I}, \mathbf{c}, \mu) - W(\mathbf{I}, \mathbf{c}, \mu_n) \right| &= \left| \sum_{j=1}^{k} p^{(j)} \int ||x(I_j) - c_j||_{dn}^2 d\mu^{(j)}(x) - \sum_{j=1}^{k} p_n^{(j)} \int ||x(I_j) - c_j||_{dn}^2 d\mu_n^{(j)}(x) \right| \nonumber \\
    &\le \left| \sum_{j=1}^{k} p^{(j)} \int ||x(I_j) - c_j||_{dn}^2 d\mu^{(j)}(x) - \sum_{j=1}^{k} p^{(j)} \int ||x(I_j) - c_j||_{dn}^2 d\mu_n^{(j)}(x) \right| \label{lem_simplify:1} \\
    &+ \left| \sum_{j=1}^{k} p^{(j)} \int ||x(I_j) - c_j||_{dn}^2 d\mu_n^{(j)}(x) - \sum_{j=1}^{k} p_n^{(j)} \int ||x(I_j) - c_j||_{dn}^2 d\mu_n^{(j)}(x) \right|. \label{lem_simplify:2}
\end{align}

Using Lemma \ref{lem:riskbound2}, we have the following bound for the first term:
\begin{align*}
    (\ref{lem_simplify:1}) &= \sum_{j=1}^{k} p^{(j)} \left| \int ||x(I_j) - c_j||_{dn}^2 d\mu^{(j)}(x) - \int ||x(I_j) - c_j||_{dn}^2 d\mu_n^{(j)}(x) \right| \\
    &\le \sum_{j=1}^{k} p^{(j)} 2\sqrt{h} \gamma(\mu^{(j)}, \mu_n^{(j)}) \le \sum_{j=1}^{k} 2\sqrt{h} \gamma(\mu^{(j)}, \mu_n^{(j)}) \le 2k\sqrt{h} \max_{1 \le j \le k} \gamma(\mu^{(j)}, \mu_n^{(j)}).
\end{align*}

We have the following bound for the second term:
\begin{equation*}
    (\ref{lem_simplify:2}) = \sum_{j=1}^{k} \int ||x(I_j) - c_j||_{dn}^2 d\mu_n^{(j)}(x) \left| p^{(j)} - p_n^{(j)} \right| \le \sum_{j=1}^{k} h \left| p^{(j)} - p_n^{(j)} \right| \le kh \max_{1 \le j \le k} \left| p^{(j)} - p_n^{(j)} \right|.
\end{equation*}

Lemma \ref{lem:riskbound3} follows directly from the above three inequalities.
\end{proof}

\begin{lemma}\label{lem:riskbound4}
For $1 \le j \le k$, $\displaystyle \lim_{n \to \infty} |p^{(j)} - p_n^{(j)}| = 0$ a.s., and $\displaystyle \lim_{n \to \infty} \gamma(\mu^{(j)}, \mu_n^{(j)}) = 0$ a.s.
\end{lemma}

\begin{proof}
It is well known that the empirical measure $\mu_n$ converges to $\mu$ almost surely. This implies that for $1 \le j \le k$, we have $p_n^{(j)}$ converges to $p^{(j)}$ almost surely, and $\mu_n^{(j)}$ converges to $\mu^{(j)}$ almost surely. For each $1 \le j \le k$, since $p_n^{(j)}$ converges to $p^{(j)}$ almost surely, we know that
\begin{equation*}
    \lim_{n \to \infty} |p^{(j)} - p_n^{(j)}| = 0\ a.s.
\end{equation*}
Since $\mu_n^{(j)}$ converges to $\mu^{(j)}$ almost surely, by Skorokhod's representation theorem, there exist $Z_n \sim \mu_n^{(j)}$ and $Z \sim \mu^{(j)}$ jointly distributed such that $Z_n \to Z$ almost surely. By the triangle inequality, we have 
\begin{equation*}
    2||Z_n||_2^2 + 2||Z||_2^2 - ||Z_n - Z||_2^2 \ge ||Z_n||_2^2 + ||Z||_2^2 - 2||Z_n||_2 ||Z||_2 \ge 0.
\end{equation*}
Hence Fatou's lemma implies
\begin{align*}
    &\liminf_{n \to \infty} \mathbb{E} \left[ 2||Z_n||_2^2 + 2||Z||_2^2 - ||Z_n - Z||_2^2 \right] \ge \mathbb{E} \left[ \liminf_{n \to \infty} \left( 2||Z_n||_2^2 + 2||Z||_2^2 - ||Z_n - Z||_2^2 \right) \right] = 4 \mathbb{E} ||Z||_2^2.
\end{align*}
Since $\lim_{n \to \infty} \mathbb{E} ||Z_n||_2^2 = \mathbb{E} ||Z||_2^2$, we must have $\lim_{n \to \infty} \mathbb{E} ||Z_n - Z||_2^2 = 0$, which implies that
\begin{equation*}
    \lim_{n \to \infty} \gamma(\mu^{(j)}, \mu_n^{(j)}) = 0\ a.s. \qedhere
\end{equation*}
\end{proof}

Having proved the above four lemmas, we are ready to prove Theorem \ref{thm:consistency}, which states the following: 

Assume that all $\mu^{(j)}$ have finite second moments that are bounded by a constant $h$:
\begin{equation*}
    \max_{1 \le j \le k} \int ||x||_2^2 d\mu^{(j)}(x) \le h.
\end{equation*}
Let $\mathbf{I}_n$ and $\mathbf{c}_n$ be a $\delta_n$-minimizer of the empirical risk. If $\lim_{n \to \infty} \delta_n = 0$, then 
\begin{equation*}
    \lim_{n \to \infty} W(\mathbf{I}_n, \mathbf{c}_n, \mu) = W^*(\mu)\ a.s.
\end{equation*}

\begin{proof}
Let $\varepsilon > 0$ be arbitrary, and let $\mathbf{I}^*$ and $\mathbf{c}^*$ be any element satisfying 
\begin{equation}\label{eq:proof1}
    \inf_\mathbf{I} \inf_\mathbf{c} W(\mathbf{I}, \mathbf{c}, \mu) \le W(\mathbf{I}^*, \mathbf{c}^*, \mu) < \inf_\mathbf{I} \inf_\mathbf{c} W(\mathbf{I}, \mathbf{c}, \mu) + \varepsilon.
\end{equation}
Then 
\begin{align*}
    W(\mathbf{I}_n, \mathbf{c}_n, \mu) - W^*(\mu) &= W(\mathbf{I}_n, \mathbf{c}_n, \mu) - \inf_\mathbf{I} \inf_\mathbf{c} W(\mathbf{I}, \mathbf{c}, \mu) \\
    &\le W(\mathbf{I}_n, \mathbf{c}_n, \mu) - (W(\mathbf{I}^*, \mathbf{c}^*, \mu) - \varepsilon) \\
    &= W(\mathbf{I}_n, \mathbf{c}_n, \mu) - W(\mathbf{I}_n, \mathbf{c}_n, \mu_n) + W(\mathbf{I}_n, \mathbf{c}_n, \mu_n) - W(\mathbf{I}^*, \mathbf{c}^*, \mu) + \varepsilon \\
    &\le W(\mathbf{I}_n, \mathbf{c}_n, \mu) - W(\mathbf{I}_n, \mathbf{c}_n, \mu_n) + (W(\mathbf{I}^*, \mathbf{c}^*, \mu_n) + \delta_n) - W(\mathbf{I}^*, \mathbf{c}^*, \mu) + \varepsilon  \\
    &\le |W(\mathbf{I}_n, \mathbf{c}_n, \mu) - W(\mathbf{I}_n, \mathbf{c}_n, \mu_n)| + |W(\mathbf{I}^*, \mathbf{c}^*, \mu_n) - W(\mathbf{I}^*, \mathbf{c}^*, \mu)| + \delta_n + \varepsilon.
\end{align*}

We now further analyze the right hand side of the last inequality:
\begin{equation}\label{eq:proof2}
    W(\mathbf{I}_n, \mathbf{c}_n, \mu) - W^*(\mu) \le |W(\mathbf{I}_n, \mathbf{c}_n, \mu) - W(\mathbf{I}_n, \mathbf{c}_n, \mu_n)| + |W(\mathbf{I}^*, \mathbf{c}^*, \mu_n) - W(\mathbf{I}^*, \mathbf{c}^*, \mu)| + \delta_n + \varepsilon.
\end{equation}

For the first term $|W(\mathbf{I}_n, \mathbf{c}_n, \mu) - W(\mathbf{I}_n, \mathbf{c}_n, \mu_n)|$, recall that $\mathbf{I}_n$ and $\mathbf{c}_n$ is a $\delta_n$-minimizer of the empirical risk. This means that for each $I_n^{(j)} \in \mathbf{I}_n$, the corresponding $c_n^{(j)} \in \mathbf{c}_n$ is selected to minimize $\int ||x(I_n^{(j)}) - c_n^{(j)}||_{dn}^2 d\mu_n^{(j)}(x)$, and it can be written as $c_n^{(j)} = \int x(I_n^{(j)}) d\mu_n^{(j)}(x)$. Note that for each $I_n^{(j)} \in \mathbf{I}_n$ and the corresponding $c_n^{(j)} \in \mathbf{c}_n$, we have
\begin{equation*}
    \int ||x(I_n^{(j)}) - c_n^{(j)}||_{dn}^2 d\mu_n^{(j)}(x) \le \int ||x(I_n^{(j)}) - c_n^{(j)}||_2^2 d\mu_n^{(j)}(x) \le \int ||x - b_n^{(j)}||_2^2 d\mu_n^{(j)}(x),
\end{equation*}
where $b_n^{(j)} = \int x d\mu_n^{(j)}(x)$. Since $\mu_n^{(j)}$ converges to $\mu^{(j)}$ almost surely, by the strong law of large numbers, we know that 
\begin{equation*}
    b_n^{(j)} = \int x d\mu_n^{(j)}(x) \overset{a.s.}{\to} \int x d\mu^{(j)}(x) = b_j,
\end{equation*}
and
\begin{equation*}
     \int ||x - b_n^{(j)}||_2^2 d\mu_n^{(j)}(x) \overset{a.s.}{\to} \int ||x - b_j||_2^2 d\mu^{(j)}(x).
\end{equation*}
Similarly, for each $I_n^{(j)} \in \mathbf{I}_n$ and the corresponding $c_n^{(j)} \in \mathbf{c}_n$, we have
\begin{equation*}
    \int ||x(I_n^{(j)}) - c_n^{(j)}||_{dn}^2 d\mu^{(j)}(x) \le \int ||x(I_n^{(j)}) - c_n^{(j)}||_2^2 d\mu^{(j)}(x) \le \int ||x - b_n^{(j)}||_2^2 d\mu^{(j)}(x),
\end{equation*}
and
\begin{equation*}
     \int ||x - b_n^{(j)}||_2^2 d\mu^{(j)}(x) \overset{a.s.}{\to} \int ||x - b_j||_2^2 d\mu^{(j)}(x).
\end{equation*}
Note that 
\begin{equation*}
    \int ||x - b_j||_2^2 d\mu^{(j)}(x) \le \int ||x||_2^2 d\mu^{(j)}(x) \le h.
\end{equation*}
Therefore, for each $j$ we can select $N_j$ such that for all $n \ge N_j$, with probability 1 we have
\begin{equation*}
    \int ||x(I_n^{(j)}) - c_n^{(j)}||_{dn}^2 d\mu_n^{(j)}(x) \le \int ||x - b_n^{(j)}||_2^2 d\mu_n^{(j)}(x) \le 2h,
\end{equation*}
and
\begin{equation*}
    \int ||x(I_n^{(j)}) - c_n^{(j)}||_{dn}^2 d\mu^{(j)}(x) \le \int ||x - b_n^{(j)}||_2^2 d\mu^{(j)}(x) \le 2h.
\end{equation*}
Let $N_0 = \max_{1 \le j \le k} N_j$, then for all $n \ge N_0$, with probability 1 we have
\begin{equation*}
    \max_{I_n^{(j)} \in \mathbf{I}_n, c_n^{(j)} \in \mathbf{c}_n} \left( \int ||x(I_n^{(j)}) - c_n^{(j)}||_{dn}^2 d\mu^{(j)}(x), \int ||x(I_n^{(j)}) - c_n^{(j)}||_{dn}^2 d\mu_n^{(j)}(x) \right) \le 2h.
\end{equation*}
Using Lemma \ref{lem:riskbound3}, for all $n \ge N_0$, with probability 1 we have
\begin{equation}\label{eq:proof3}
    |W(\mathbf{I}_n, \mathbf{c}_n, \mu) - W(\mathbf{I}_n, \mathbf{c}_n, \mu_n)| \le 2k\sqrt{2h} \max_{1 \le j \le k} \gamma(\mu^{(j)}, \mu_n^{(j)}) + 2kh \max_{1 \le j \le k} \left| p^{(j)} - p_n^{(j)} \right|.
\end{equation}

For the second term $|W(\mathbf{I}^*, \mathbf{c}^*, \mu_n) - W(\mathbf{I}^*, \mathbf{c}^*, \mu)|$, for each $I_j \in \mathbf{I}^*$, we can write the corresponding $c_j \in \mathbf{c}^*$ as $c_j = \int x(I_j) d\mu^{(j)}(x)$ in order to minimize $W(\mathbf{I}^*, \mathbf{c}^*, \mu)$. Similar to the steps in bounding the first term, for each $I_j \in \mathbf{I}^*$ and the corresponding $c_j \in \mathbf{c}^*$, we have
\begin{equation*}
    \int ||x(I_j) - c_j||_{dn}^2 d\mu_n^{(j)}(x) \le \int ||x(I_j) - c_j||_2^2 d\mu_n^{(j)}(x) \le \int ||x - b_j||_2^2 d\mu_n^{(j)}(x) \overset{a.s.}{\to} \int ||x - b_j||_2^2 d\mu^{(j)}(x),
\end{equation*}
and
\begin{equation*}
    \int ||x(I_j) - c_j||_{dn}^2 d\mu^{(j)}(x) \le \int ||x(I_j) - c_j||_2^2 d\mu^{(j)}(x) \le \int ||x - b_j||_2^2 d\mu^{(j)}(x).
\end{equation*}
Therefore, for each $j$ we can select $M_j$ such that for all $n \ge M_j$, with probability 1 we have
\begin{equation*}
    \int ||x(I_j) - c_j||_{dn}^2 d\mu_n^{(j)}(x) \le \int ||x - b_j||_2^2 d\mu_n^{(j)}(x) \le 2h,
\end{equation*}
and
\begin{equation*}
    \int ||x(I_j) - c_j||_{dn}^2 d\mu^{(j)}(x) \le \int ||x - b_j||_2^2 d\mu^{(j)}(x) \le 2h.
\end{equation*}
Let $M_0 = \max_{1 \le j \le k} M_j$, then for all $n \ge M_0$, with probability 1 we have
\begin{equation*}
    \max_{I_j \in \mathbf{I}^*, c_j \in \mathbf{c}^*} \left( \int ||x(I_j) - c_j||_{dn}^2 d\mu^{(j)}(x), \int ||x(I_j) - c_j||_{dn}^2 d\mu_n^{(j)}(x) \right) \le 2h.
\end{equation*}
Using Lemma \ref{lem:riskbound3}, for all $n \ge M_0$, with probability 1 we have
\begin{equation}\label{eq:proof4}
    |W(\mathbf{I}^*, \mathbf{c}^*, \mu_n) - W(\mathbf{I}^*, \mathbf{c}^*, \mu)| \le 2k\sqrt{2h} \max_{1 \le j \le k} \gamma(\mu^{(j)}, \mu_n^{(j)}) + 2kh \max_{1 \le j \le k} \left| p^{(j)} - p_n^{(j)} \right|.
\end{equation}

Combining the inequalities (\ref{eq:proof2}), (\ref{eq:proof3}), (\ref{eq:proof4}), for all $n \ge \max(N_0, M_0)$, with probability 1 we have
\begin{equation*}
    W(\mathbf{I}_n, \mathbf{c}_n, \mu) - W^*(\mu) \le 4k\sqrt{2h} \max_{1 \le j \le k} \gamma(\mu^{(j)}, \mu_n^{(j)}) + 4kh \max_{1 \le j \le k} \left| p^{(j)} - p_n^{(j)} \right| + \delta_n + \varepsilon.
\end{equation*}
Using Lemma \ref{lem:riskbound4}, we know that $\max_{1 \le j \le k} \gamma(\mu^{(j)}, \mu_n^{(j)}) \overset{a.s.}{\to} 0$, and $\max_{1 \le j \le k} \left| p^{(j)} - p_n^{(j)} \right| \overset{a.s.}{\to} 0$. Since $\varepsilon$ is arbitrary and $\lim_{n \to \infty} \delta_n = 0$, we have
\begin{equation*}
    W(\mathbf{I}_n, \mathbf{c}_n, \mu) - W^*(\mu) \overset{a.s.}{\to} 0. 
    \qedhere
\end{equation*}
\end{proof}

\begin{remark}
If we slightly modify inequality (\ref{eq:proof1}) and choose $\varepsilon_n$ such that $\lim_{n \to \infty} \varepsilon_n = 0$, and let $\mathbf{I}_n^*$ and $\mathbf{c}_n^*$ be any element satisfying 
\begin{equation}\label{eq:proof5}
    \inf_\mathbf{I} \inf_\mathbf{c} W(\mathbf{I}, \mathbf{c}, \mu) \le W(\mathbf{I}_n^*, \mathbf{c}_n^*, \mu) < \inf_\mathbf{I} \inf_\mathbf{c} W(\mathbf{I}, \mathbf{c}, \mu) + \varepsilon_n.
\end{equation}
Then, similar to inequality (\ref{eq:proof2}), we could obtain
\begin{equation}\label{eq:proof6}
    W(\mathbf{I}_n, \mathbf{c}_n, \mu) - W^*(\mu) \le |W(\mathbf{I}_n, \mathbf{c}_n, \mu) - W(\mathbf{I}_n, \mathbf{c}_n, \mu_n)| + |W(\mathbf{I}_n^*, \mathbf{c}_n^*, \mu_n) - W(\mathbf{I}_n^*, \mathbf{c}_n^*, \mu)| + \delta_n + \varepsilon_n.
\end{equation}
Also, similar to inequality (\ref{eq:proof4}), we could prove that for all $n \ge M_0$, with probability 1 we have
\begin{equation}\label{eq:proof7}
    |W(\mathbf{I}_n^*, \mathbf{c}_n^*, \mu_n) - W(\mathbf{I}_n^*, \mathbf{c}_n^*, \mu)| \le 2k\sqrt{2h} \max_{1 \le j \le k} \gamma(\mu^{(j)}, \mu_n^{(j)}) + 2kh \max_{1 \le j \le k} \left| p^{(j)} - p_n^{(j)} \right|.
\end{equation}
Combining the inequalities (\ref{eq:proof3}), (\ref{eq:proof6}), (\ref{eq:proof7}), for all $n \ge \max(N_0, M_0)$, with probability 1 we have
\begin{equation*}
    W(\mathbf{I}_n, \mathbf{c}_n, \mu) - W^*(\mu) \le 4k\sqrt{2h} \max_{1 \le j \le k} \gamma(\mu^{(j)}, \mu_n^{(j)}) + 4kh \max_{1 \le j \le k} \left| p^{(j)} - p_n^{(j)} \right| + \delta_n + \varepsilon_n.
\end{equation*}
Now, if we assume there exist $\beta_n$ such that for $1 \le j \le k$, we have
\begin{equation}\label{eq:proof8}
    \lim_{n \to \infty} \beta_n |p^{(j)} - p_n^{(j)}| < \infty \text{ a.s., and } \lim_{n \to \infty} \beta_n \gamma(\mu^{(j)}, \mu_n^{(j)}) < \infty \text{ a.s.} 
\end{equation}
Then, as long as we choose $\delta_n$ and $\varepsilon_n$ such that $\lim_{n \to \infty} \beta_n \delta_n < \infty$ and $\lim_{n \to \infty} \beta_n \varepsilon_n < \infty$, we have
\begin{align*}
    & \lim_{n \to \infty} \beta_n \left[ W(\mathbf{I}_n, \mathbf{c}_n, \mu) - W^*(\mu) \right] \\
    &\le 4k\sqrt{2h} \max_{1 \le j \le k} \lim_{n \to \infty} \beta_n \gamma(\mu^{(j)}, \mu_n^{(j)}) + 4kh \max_{1 \le j \le k} \lim_{n \to \infty} \beta_n \left| p^{(j)} - p_n^{(j)} \right| + \lim_{n \to \infty} \beta_n \delta_n + \lim_{n \to \infty} \beta_n \varepsilon_n \\
    &\le \infty \text{ a.s.}
\end{align*}
There are a few choices for $\beta_n$ that satisfy inequality (\ref{eq:proof8}), depending on the assumptions on $\mu_j$ and $p_j$. The earliest result is by \citet{ajtai1984optimal} for the Lebesgue measure which was later sharpened by \citet{dobric1995} to $\beta_n = n^{1/p}$. Theorem 11.1.6 of \citet{rachev1991probability} generalizes \citet{dudley1969speed} to show that under a metric entropy condition, we could let $\beta_n = n^{2/p}$. \citet{horowitz1994mean} proved that under some very weak assumptions, we could let $\beta_n = n^{2/(p+4)}$. Some refinements were provided in \citet{fournier2015rate} and \citet{weed2019sharp}.
\end{remark}

\section{Algorithm}\label{sec:algorithm}
In this section, we present a simple algorithm that outputs the $k$ disjoint subsets of features $I_1, \ldots, I_k$. The idea is to first transpose the $n \times p$ data matrix and then use an adapted version of the $k$-means clustering algorithm to produce a partition of the $p$ rows $I_1, \ldots, I_k$. The algorithm works as shown in Algorithm \ref{algorithm1}.

\begin{algorithm}[ht]
\begin{enumerate}
    \item Start by transposing the $n \times p$ data matrix and then performing $k$-means clustering on the rows to obtain $k$ clusters as the initial partition of the $p$ rows $I_1, \ldots, I_k$.
    \item We use an alternating procedure to update $I_1, \ldots, I_k$. It is quite similar to the Lloyd's algorithm in $k$-means clustering, except that the cluster centers $m_j$ are defined on $\mathbb{R}^{S_j}$ instead of $\mathbb{R}^{n}$, and the distance function is induced by the dimensionality-normalized norm instead of the Euclidean norm: 
    \begin{enumerate}
        \item (Update step) Given row partitions $I_1, \ldots, I_k$, update the cluster centers $m_1, \ldots, m_k$ by
        \begin{equation*}
            m_j = \frac{1}{|I_j|} \sum_{i \in I_j} T_i(S_j), 1 \le j \le k,
        \end{equation*}
        where $T_i$ denotes the $i$-th row of the transposed $p \times n$ data matrix. 
        \item (Assignment step) Given cluster centers $m_1, \ldots, m_k$, update the row partitions $I_1, \ldots, I_k$ by assigning every row to the cluster center with the smallest distance (induced by the dimensionality-normalized norm), and all the rows that are closest to $m_j$ form $I_j, 1 \le j \le k$.
    \end{enumerate}
    Alternate between (a) and (b) until convergence, and obtain a partition of the $p$ rows $I_1, \ldots, I_k$.
\end{enumerate}
\caption{Finding the $k$ disjoint subsets of features $I_1, \ldots, I_k$ (no feature selection)}
\label{algorithm1}
\end{algorithm}

Recall that $S_j$ denote the set of indices of training samples that belong to class $j$. After we have obtained the $k$ disjoint subsets of features $I_1, \ldots, I_k$, we can compute the $k$ disjoint centroids that represent the $k$ classes using the following equation:
\begin{equation}\label{ndc_centroid}
    c_j = \frac{1}{|S_j|} \sum_{i \in S_j} X_i(I_j),\ 1 \le j \le k.
\end{equation}
To classify a new data point, we simply choose the class with the nearest disjoint centroid (distance induced by the dimensionality-normalized norm).

In practice, it is recommended to run our algorithm multiple times to produce multiple partitions, and choose the partition that results in the lowest training error. There are mainly two reasons:
\begin{enumerate}
    \item First, our algorithm depends on the partition obtained from the initial $k$-means clustering, which itself is not deterministic and might provide different results in different runs. Therefore, running our algorithm multiple times increases the probability of finding a partition that gives better performance.
    \item Second, just like $k$-means clustering algorithm, our algorithm also might encounter empty cluster problem, although the probability is small when $k$ is much smaller than $p$. More specifically, if in any assignment step any group of row indices $I_j, 1 \le j \le k$ becomes empty, then the algorithm cannot proceed and need to restart.
\end{enumerate}

\section{Feature Selection}\label{sec:selection}
In this section, we consider extending Algorithm \ref{algorithm1} to perform feature selection. So far, although our method partitions features into disjoint subsets, it still uses all the features to classify new data points. However, in many applications such as gene expression analysis, the data usually include thousands of genes, many of which can be considered as irrelevant to predicting certain types of cancer. In those situations, it would be desirable if our method could perform feature selection, namely selecting a subset of the features so that only those features are used in making predictions.  

In order to enable our method to perform feature selection, we need to make some small adjustments to Algorithm \ref{algorithm1}. More specifically, we add a new special cluster $I_0$, and the features that belong to this special cluster are not used in prediction. Its cluster center $m_0$ are defined on $\mathbb{R}^{S_0}$, and we define $S_0 = \{1, \ldots, n\}$. Intuitively, this cluster center represents a ``global'' baseline that is based on all data points, as opposed to other cluster centers $m_j$ that represent class-specific baselines that are based on data points specific to class $j$. When classifying new data points, we still choose the class with the nearest disjoint centroid among $c_1, \ldots, c_k$ (distance induced by the dimensionality-normalized norm), which in turn depends on features in $I_1, \ldots, I_k$. In this way, the features in $I_0$ are not used at all in making predictions, and therefore our method can be considered as performing feature selection.

\begin{algorithm}
\begin{enumerate}
    \item Start by transposing the $n \times p$ data matrix and then performing $k$-means clustering on the rows to obtain $k+1$ clusters as the initial partition of the $p$ rows $I_0, \ldots, I_k$.
    \item We use an alternating procedure to update $I_0, \ldots, I_k$. It is quite similar to the Lloyd's algorithm in $k$-means clustering, except that the cluster centers $m_j$ are defined on $\mathbb{R}^{S_j}$ instead of $\mathbb{R}^{n}$, and the distance function is induced by the dimensionality-normalized norm instead of the Euclidean norm: 
    \begin{enumerate}
        \item (Update step) Given row partitions $I_0, \ldots, I_k$, update the cluster centers $m_0, \ldots, m_k$ by
        \begin{equation*}
            m_j = \frac{1}{|I_j|} \sum_{i \in I_j} T_i(S_j), 0 \le j \le k,
        \end{equation*}
        where $T_i$ denotes the $i$-th row of the transposed $p \times n$ data matrix. 
        \item (Assignment step) Given cluster centers $m_0, \ldots, m_k$, update the row partitions $I_0, \ldots, I_k$ by assigning every row to the cluster center with the smallest distance (induced by the dimensionality-normalized norm), and all the rows that are closest to $m_j$ form $I_j, 0 \le j \le k$. Note that the distance to $m_0$ is multiplied by a factor $\lambda$, which is a tuning parameter.
    \end{enumerate}
    Alternate between (a) and (b) until convergence, and obtain a partition of the $p$ rows $I_0, \ldots, I_k$.
\end{enumerate}
\caption{Finding the $k+1$ disjoint subsets of features $I_0, \ldots, I_k$ (with feature selection)}
\label{algorithm2}
\end{algorithm}

The modified algorithm that is able to perform feature selection works as shown in Algorithm \ref{algorithm2}. As a way to control the number of selected features, we introduce a tuning parameter $\lambda$. When computing the distances (induced by the dimensionality-normalized norm) to the clusters centers in the assignment step, the distance to $m_0$ is multiplied by a factor $\lambda$. When $\lambda = \infty$, the distance to $m_0$ will become $\infty$, so no feature will be assigned to the special cluster $I_0$, which means that all features will be selected. As $\lambda$ gets smaller and smaller, the distance to $m_0$ will become smaller and smaller, so features are more and more likely to be assigned to the special cluster $I_0$, which means that less and less features are getting selected. In practice, $\lambda$ can be considered as a hyperparameter that needs to be tuned based on the data, because it affects both the number of selected features and the performance of the model. 

After we have obtained the $k+1$ disjoint subsets of features $I_0, \ldots, I_k$, the rest of the process is exactly the same as before: compute the $k$ disjoint centroids $c_1, \ldots, c_k$ based on $I_1, \ldots, I_k$ using Equation (\ref{ndc_centroid}), and classify a new data point to the class with the nearest disjoint centroid (distance induced by the dimensionality-normalized norm). Noticeably, the features in $I_0$ are not involved in the computation of the distances to the $k$ disjoint centroids $c_1, \ldots, c_k$, therefore they are not used in prediction.

\section{Simulation Studies}\label{sec:simulation}
In this section, we evaluate and compare the performance of seven classification methods on simulated data with different settings. The first three classification methods are all based on nearest centroid:
\begin{enumerate}
    \item Nearest disjoint centroid (NDC): This is the method presented in this paper. We consider both versions of our nearest disjoint centroid method, with and without feature selection. Algorithm \ref{algorithm1} (NDC) is the version without feature selection, and Algorithm \ref{algorithm2} (NDC-S) is the version with feature selection. For both algorithms, we run 100 times and pick the best partition, as suggested at the end of Section \ref{sec:algorithm}.
    \item Nearest centroid (NC): This method simply classifies every data point to the class with the nearest centroid, and each centroid is defined as the average of the data points that belong to each class.
    \item Nearest shrunken centroid (NSC) \citep{tibshirani2002diagnosis}: This method is a simple modification of the nearest centroid method. It shrinks the class centroids toward the overall centroid after standardizing by the within-class standard deviation. The shrinkage process achieves feature selection.
\end{enumerate}
In addition, we also include the following four widely used classification methods: $k$-nearest neighbors (KNN) \citep{cover1967nearest}, linear discriminant analysis (LDA) \citep{fisher1936use}, support vector machine (SVM) \citep{cortes1995support}, and logistic regression with $L_1$ regularization (Logistic) \citep{hastie2009elements}. The evaluation metric is the misclassification rate. The number of neighbors in KNN is set to be 15. Other hyperparameters, including the $\lambda$ in NDC-S, the threshold $\Delta$ in NSC, and the $\lambda$ in logistic regression with $L_1$ regularization, are chosen via cross-validation on the training set.

We generate simulated data in four different settings. In all settings, the data matrix satisfies the following property:
\begin{enumerate}
    \item The rows have $k$ blocks, each of size $n$. They represent $k$ classes, each consisting of $n$ data points.
    \item The columns also have $k$ blocks, each of size $d$. They represent $k$ groups of features, each of size $d$.
    \item As a whole, the data matrix consists of $k \times k$ blocks, each of size $n \times d$.
    \item All the entries in the data matrix are independent. In addition, the entries in each small $n \times d$ block are identically distributed. The entries in the $k$ blocks on the main diagonal of the data matrix follow $\mathcal{N}(\mu_1, \sigma_1^2)$, and all the other entries in the data matrix follow $\mathcal{N}(\mu_2, \sigma_2^2)$.
\end{enumerate}
In all settings, we set $k = 4$ and $n = 250$. The rest of the parameters, including $d$, $\mu_1$, $\mu_2$, $\sigma_1$, $\sigma_2$, might vary in different settings. For each setting, we perform 50 simulations, and report the means and standard errors of the misclassification rates. In each simulation, we generate two data matrices, one as training set and the other as test set.

\subsection{Simulation 1: Blocks with Different Means and the Same Variance}
In the first simulation, we consider the case where the $k \times k$ blocks have different means and the same variance. More specifically, we set $\mu_1 = a$, and $\mu_2 = 0$, where $a \in \{0.3, 0.6, 0.9\}$. As $a$ increases, the difference between the means in different blocks also increases. In addition, we set $\sigma_1 = \sigma_2 = 1$, which means that all entries have the same standard deviation of 1. We also set $d \in \{3, 5, 10\}$. As $d$ increases, the number of features in each block also increases.

\begin{table}[ht]
\centering
\begin{tabular}{ c c c c c c c c c }
    \hline
    & NDC & NDC-S & NC & NSC & KNN & LDA & SVM & Logistic \\
    \hline
    & & & & $d = 3$ & & & & \\
    \hline
    $a = 0.3$ & 0.736 & 0.738 & 0.611 & 0.614 & 0.678 & 0.612 & 0.635 & 0.609 \\
    $a = 0.6$ & 0.669 & 0.686 & 0.447 & 0.448 & 0.511 & 0.448 & 0.467 & 0.445 \\
    $a = 0.9$ & 0.542 & 0.586 & 0.286 & 0.286 & 0.332 & 0.288 & 0.306 & 0.286 \\
    \hline
    & & & & $d = 5$ & & & & \\
    \hline
    $a = 0.3$ & 0.731 & 0.737 & 0.570 & 0.573 & 0.651 & 0.571 & 0.590 & 0.569 \\
    $a = 0.6$ & 0.626 & 0.648 & 0.350 & 0.351 & 0.433 & 0.354 & 0.373 & 0.351 \\
    $a = 0.9$ & 0.475 & 0.520 & 0.178 & 0.180 & 0.227 & 0.182 & 0.196 & 0.182 \\
    \hline
    & & & & $d = 10$ & & & & \\
    \hline
    $a = 0.3$ & 0.724 & 0.730 & 0.488 & 0.491 & 0.607 & 0.494 & 0.507 & 0.489 \\
    $a = 0.6$ & 0.564 & 0.609 & 0.210 & 0.211 & 0.301 & 0.217 & 0.224 & 0.214 \\
    $a = 0.9$ & 0.346 & 0.490 & 0.058 & 0.058 & 0.087 & 0.063 & 0.066 & 0.064 \\
    \hline
\end{tabular}
\caption{The means of the misclassification rate for Simulation 1 over 50 simulations. Most of the standard errors are less than 0.003, and the largest standard error is 0.011.}
\label{tab:sim1}
\end{table}

Results are reported in Table \ref{tab:sim1}. In this setting, we see that NC, NSC, LDA, and Logistic have extremely similar and also the smallest misclassification rates, followed closely by SVM and KNN, and finally NDC and NDC-S with significantly larger misclassification rates. In addition, we observe a general pattern that as $d$ and $a$ increase, the misclassification rates decrease. This pattern makes intuitive sense, because larger $d$ means more features, and larger $a$ means larger difference between the means in different blocks, both of which should improve the performance of classification methods.

It is important to point out that our method (NDC and NDC-S) performing worse than other competing classification methods, including NC and NSC which are directly comparable, is expected in this setting. The reason is that all the features provide useful signals by having different means across different classes, and there is no benefit in considering features separately since they all have the same variance. Our method defines centroids using disjoint subsets of features, thereby losing valuable information compared to NC and NSC, both of which define centroids using all the features. However, when different features have different variances, our method would perform much better that all other classification methods, as we will see in the following simulations.

\subsection{Simulation 2: Blocks with Different Variances and the Same Mean}
In the second simulation, we consider the case where the $k \times k$ blocks have different variances and the same mean. More specifically, we set $\sigma_1 = 1$, and $\sigma_2 = 1+b$, where $b \in \{0.3, 0.6, 0.9\}$. As $b$ increases, the difference between the variances in different blocks also increases. In addition, we set $\mu_1 = \mu_2 = 0$, which means that all entries have the same mean of 0. We also set $d \in \{3, 5, 10\}$. As $d$ increases, the number of features in each block also increases.

\begin{table}[ht]
\centering
\begin{tabular}{ c c c c c c c c c }
    \hline
    & NDC & NDC-S & NC & NSC & KNN & LDA & SVM & Logistic \\
    \hline
    & & & & $d = 3$ & & & & \\
    \hline
    $b = 0.3$ & 0.699 & 0.696 & 0.748 & 0.748 & 0.707 & 0.747 & 0.680 & 0.749 \\
    $b = 0.6$ & 0.528 & 0.610 & 0.739 & 0.738 & 0.638 & 0.739 & 0.583 & 0.746 \\
    $b = 0.9$ & 0.357 & 0.396 & 0.739 & 0.737 & 0.570 & 0.739 & 0.501 & 0.750 \\
    \hline
    & & & & $d = 5$ & & & & \\
    \hline
    $b = 0.3$ & 0.583 & 0.638 & 0.749 & 0.747 & 0.705 & 0.750 & 0.675 & 0.750 \\
    $b = 0.6$ & 0.358 & 0.378 & 0.747 & 0.743 & 0.630 & 0.747 & 0.555 & 0.747 \\
    $b = 0.9$ & 0.229 & 0.246 & 0.740 & 0.736 & 0.555 & 0.741 & 0.458 & 0.748 \\
    \hline
    & & & & $d = 10$ & & & & \\
    \hline
    $b = 0.3$ & 0.426 & 0.446 & 0.746 & 0.742 & 0.700 & 0.744 & 0.686 & 0.748 \\
    $b = 0.6$ & 0.189 & 0.201 & 0.743 & 0.739 & 0.623 & 0.743 & 0.583 & 0.749 \\
    $b = 0.9$ & 0.075 & 0.080 & 0.736 & 0.730 & 0.547 & 0.738 & 0.480 & 0.744 \\
    \hline
\end{tabular}
\caption{The means of the misclassification rate for Simulation 2 over 50 simulations. Most of the standard errors are less than 0.003, and the largest standard error is 0.008.}
\label{tab:sim2}
\end{table}

\newpage

Results are reported in Table \ref{tab:sim2}. In this setting, we see that NC, NSC, LDA, and Logistic also have extremely similar but the worst performance, with misclassification rates around 0.75 (equivalent to random guessing) in all cases. These numbers indicate that NC, NSC, LDA, and Logistic are all unable to detect heteroskedastic structure in the data, regardless of the value of $d$ and $b$. SVM and KNN perform slightly better, although their misclassification rates are still much larger than those of NDC and NDC-S, both of which clearly outperform all other competing classification methods.

The reason that our method performs well in this setting lies in the fact that different features have different variances across different classes. Although different features have the same mean so the centroids have the same value, by defining centroids using disjoint subsets of features, different variances across different classes lead to different distances to different centroids. In addition, Corollary \ref{cor:minimizer} guarantees that when $\sigma_1 < \sigma_2$, we can obtain the appropriate disjoint subsets of features.

\subsection{Simulation 3: Blocks with Different Means and Different Variances}
In the third simulation, we consider the case where the $k \times k$ blocks have different means and different variances, which is a combination of the first and second case. More specifically, we set $\mu_1 = c, \sigma_1 = 1$, and $\mu_2 = 0, \sigma_2 = 1+c$, where $c \in \{0.3, 0.6, 0.9\}$. As $c$ increases, the difference between the means and variances in different blocks also increases. We also set $d \in \{3, 5, 10\}$. As $d$ increases, the number of features in each block also increases.

\begin{table}[ht]
\centering
\begin{tabular}{ c c c c c c c c c }
    \hline
    & NDC & NDC-S & NC & NSC & KNN & LDA & SVM & Logistic \\
    \hline
    & & & & $d = 3$ & & & & \\
    \hline
    $c = 0.3$ & 0.670 & 0.680 & 0.664 & 0.666 & 0.668 & 0.664 & 0.630 & 0.666 \\
    $c = 0.6$ & 0.466 & 0.537 & 0.579 & 0.583 & 0.540 & 0.579 & 0.496 & 0.579 \\
    $c = 0.9$ & 0.293 & 0.394 & 0.511 & 0.514 & 0.439 & 0.513 & 0.397 & 0.510 \\
    \hline
    & & & & $d = 5$ & & & & \\
    \hline
    $c = 0.3$ & 0.551 & 0.613 & 0.633 & 0.637 & 0.649 & 0.635 & 0.596 & 0.633 \\
    $c = 0.6$ & 0.301 & 0.324 & 0.511 & 0.516 & 0.490 & 0.513 & 0.426 & 0.508 \\
    $c = 0.9$ & 0.162 & 0.181 & 0.412 & 0.412 & 0.360 & 0.416 & 0.311 & 0.408 \\
    \hline
    & & & & $d = 10$ & & & & \\
    \hline
    $c = 0.3$ & 0.388 & 0.425 & 0.563 & 0.566 & 0.612 & 0.567 & 0.535 & 0.564 \\
    $c = 0.6$ & 0.136 & 0.146 & 0.378 & 0.381 & 0.393 & 0.388 & 0.326 & 0.378 \\
    $c = 0.9$ & 0.037 & 0.041 & 0.251 & 0.252 & 0.239 & 0.261 & 0.201 & 0.250 \\
    \hline
\end{tabular}
\caption{The means of the misclassification rate for Simulation 3 over 50 simulations. Most of the standard errors are less than 0.003, and the largest standard error is 0.009.}
\label{tab:sim3}
\end{table}

Results are reported in Table \ref{tab:sim3}. In this setting, we see that in general, the ranking of different classification methods is similar to the ranking in Simulation 2: NDC and NDC-S clearly have the best performance, followed by SVM and KNN, and NC, NSC, LDA, and Logistic have similar but the worst performance. Comparing the results in Table \ref{tab:sim3} to those in Table \ref{tab:sim1}, we notice that the additional difference between the variances in different blocks significantly helps NDC and NDC-S, leading to much smaller misclassification rates. In contrast, the misclassification rates of all other classification methods increase significantly after introducing the additional difference between block variances. Importantly, in this simulation, larger $c$ means larger difference between both the means and the variances in different blocks, so both kinds of signals are present in the data. In this situation, NDC and NDC-S outperform all other competing classifiers, which indicates that our method could potentially obtain competitive performance when dealing with complex datasets in the real world.

\subsection{Simulation 4: Adding Irrelevant Features}
In the fourth simulation, we study the impact of adding irrelevant features on different classification methods. More specifically, we fix $d=5$, and the first 20 columns of the data matrix is the same as the data matrix in Simulation 3, where we set $\mu_1 = c, \sigma_1 = 1$, $\mu_2 = 0, \sigma_2 = 1+c$, and $c \in \{0.3, 0.6, 0.9\}$. However, the remaining $r$ columns of the data matrix are $r$ irrelevant features consisting of i.i.d.\ standard Gaussian variables, where $r \in \{20, 40, 80\}$.

\begin{table}[ht]
\centering
\begin{tabular}{ c c c c c c c c c }
    \hline
    & NDC & NDC-S & NC & NSC & KNN & LDA & SVM & Logistic \\
    \hline
    & & & & $r = 20$ & & & & \\
    \hline
    $c = 0.3$ & 0.548 & 0.580 & 0.633 & 0.634 & 0.668 & 0.642 & 0.625 & 0.632 \\
    $c = 0.6$ & 0.349 & 0.330 & 0.512 & 0.514 & 0.519 & 0.527 & 0.490 & 0.512 \\
    $c = 0.9$ & 0.229 & 0.168 & 0.414 & 0.414 & 0.386 & 0.431 & 0.387 & 0.412 \\
    \hline
    & & & & $r = 40$ & & & & \\
    \hline
    $c = 0.3$ & 0.571 & 0.582 & 0.639 & 0.640 & 0.682 & 0.651 & 0.643 & 0.638 \\
    $c = 0.6$ & 0.380 & 0.322 & 0.511 & 0.511 & 0.545 & 0.536 & 0.514 & 0.511 \\
    $c = 0.9$ & 0.262 & 0.164 & 0.411 & 0.414 & 0.406 & 0.445 & 0.412 & 0.412 \\
    \hline
    & & & & $r = 80$ & & & & \\
    \hline
    $c = 0.3$ & 0.602 & 0.608 & 0.647 & 0.643 & 0.696 & 0.665 & 0.659 & 0.645 \\
    $c = 0.6$ & 0.415 & 0.310 & 0.521 & 0.515 & 0.572 & 0.560 & 0.540 & 0.513 \\
    $c = 0.9$ & 0.302 & 0.162 & 0.418 & 0.417 & 0.439 & 0.467 & 0.442 & 0.415 \\
    \hline
\end{tabular}
\caption{The means of the misclassification rate for Simulation 4 over 50 simulations. Most of the standard errors are less than 0.003, and the largest standard error is 0.005.}
\label{tab:sim4a}
\end{table}

The misclassification rates are reported in Table \ref{tab:sim4a}. Comparing to the results for $d=5$ in Table \ref{tab:sim3}, we see that NDC-S, NC, NSC, and Logistic seem to be only minimally affected by the presence of irrelevant features. However, other classification methods, including NDC, KNN, LDA, and SVM, are all noticeably affected by the inclusion of irrelevant features, and their misclassification rates further increase as $r$ increases. In particular, the difference between the behaviors of NDC and NDC-S in this setting demonstrates that by including feature selection as part of the algorithm, NDC-S becomes much more robust to the presence of irrelevant features. For datasets in the real world, it is often the case that some of the features are irrelevant, and therefore NDC-S might be a better default choice to use on real-world data.

\begin{table}[ht]
\centering
\begin{tabular}{ c c c c c c c c c }
    \hline
    & NDC & NDC-S & NC & NSC & KNN & LDA & SVM & Logistic \\
    \hline
    & & & & $r = 20$ & & & & \\
    \hline
    $c = 0.3$ & 40(0) & 31(2) & 40(0) & 32(1) & 40(0) & 40(0) & 40(0) & 33(0) \\
    $c = 0.6$ & 40(0) & 24(1) & 40(0) & 27(1) & 40(0) & 40(0) & 40(0) & 34(0) \\
    $c = 0.9$ & 40(0) & 21(1) & 40(0) & 23(1) & 40(0) & 40(0) & 40(0) & 35(0) \\
    \hline
    & & & & $r = 40$ & & & & \\
    \hline
    $c = 0.3$ & 60(0) & 45(3) & 60(0) & 41(2) & 60(0) & 60(0) & 60(0) & 41(1) \\
    $c = 0.6$ & 60(0) & 22(2) & 60(0) & 34(2) & 60(0) & 60(0) & 60(0) & 43(1) \\
    $c = 0.9$ & 60(0) & 21(1) & 60(0) & 31(2) & 60(0) & 60(0) & 60(0) & 45(1) \\
    \hline
    & & & & $r = 80$ & & & & \\
    \hline
    $c = 0.3$ & 100(0) & 80(5) & 100(0) & 54(4) & 100(0) & 100(0) & 100(0) & 51(1) \\
    $c = 0.6$ & 100(0) & 23(2) & 100(0) & 39(4) & 100(0) & 100(0) & 100(0) & 57(1) \\
    $c = 0.9$ & 100(0) & 20(0) & 100(0) & 34(4) & 100(0) & 100(0) & 100(0) & 61(1) \\
    \hline
\end{tabular}
\caption{The means (and standard errors) of the number of selected features for Simulation 4 over 50 simulations.}
\label{tab:sim4b}
\end{table}

To validate our hypothesis that the feature selection part of NDC-S is working as intended, we also compute the means and standard errors of the number of selected features for different classification methods, and the results are reported in Table \ref{tab:sim4b}. As we can see, other than NDC-S, NSC, and Logistic, the remaining five classification methods always use all the features, because they are not capable of performing feature selection. Comparing the feature selection of NDC-S, NSC, and Logistic, we could argue that in general NDC-S has the best performance. This is because for $c = 0.6$ and $c = 0.9$, regardless of the number of irrelevant features $r$, NDC-S always select close to 20 features, which is exactly the number of relevant features in the data.

\section{Applications}\label{sec:applications}
In this section, we apply our method to three gene expression datasets, all of which were proposed and preprocessed by \citet{de2008clustering}. In all three datasets, the rows represent different samples of tissues, and the columns represent different genes. The samples have already been labeled with different classes based on their types of tissue. We evaluate and compare the performance of the same eight classification methods: nearest disjoint centroid classifier without feature selection (NDC), nearest disjoint centroid classifier with feature selection (NDC-S), nearest centroid classifier (NC), nearest shrunken centroid classifier (NSC), $k$-nearest neighbors (KNN), linear discriminant analysis (LDA), support vector machine (SVM), and logistic regression with $L_1$ regularization (Logistic). We perform 3-fold cross validation on the datasets, and report the means and standard errors of the misclassification rates.

In addition to misclassification rates, we also consider whether the classifiers can perform feature selection, and if yes, how many features are selected. We know that NDC, NC, KNN, LDA, and SVM require all the features to perform classification. However, a varying number of features can be selected by changing the threshold $\Delta$ in NSC, or changing the $\lambda$ in NDC-S or Logistic. Those hyperparameters are selected by nested cross-validation to achieve the smallest misclassification rate on each fold, and we also report the means and the standard errors of the number of features selected by each classifier.

\subsection{Breast and Colon Cancer Gene Expression Dataset}
The first dataset consists of 104 samples and 182 genes. There are two types of samples: 62 samples correspond to breast cancer tissues, and 42 samples correspond to colon cancer tissues.

\begin{table}[ht]
\centering
\resizebox{\columnwidth}{!}{
\begin{tabular}{ c c c c c c c c }
    \hline
    NDC & NDC-S & NC & NSC & KNN & LDA & SVM & Logistic \\
    \hline
    0.029(0.029) & 0.019(0.010) & 0.183(0.041) & 0.125(0.038) & 0.087(0.001) & 0.058(0.017) & 0.048(0.009) & 0.019(0.010) \\
    \hline
\end{tabular}
}
\caption{The means (and standard errors) of the misclassification rates on the breast and colon cancer gene expression dataset.}
\label{tab:bc_errors}
\end{table}

The misclassification rates are reported in Table \ref{tab:bc_errors}. As we can see, NDC-S and Logistic have the best performance, followed closely in turn by NDC, SVM, LDA, and KNN. Finally, NC and NSC have the worst performance, with misclassification rates over 12\%. In particular, the small standard errors indicate that the difference between the performance of our method (NDC and NDC-S) and the two directly comparable classifiers (NC and NSC) is quite significant. One possible reason that our method performs well on this dataset is that there is a natural interpretation for the disjoint features that our method produced: two disjoint groups of genes that are useful for identifying breast and colon cancer, respectively. Since breast cancer and colon cancer are two completely different types of cancer, it is quite possible that the genes that are useful in predicting one type of cancer are largely irrelevant to predicting another type of cancer. Therefore, our nearest disjoint centroid classifiers, which identify two disjoint sets of genes that are used in predicting the two types of cancer, perform better than the nearest centroid classifier and the nearest shrunken centroid classifier, both of which rely on the same set of genes to predict the two types of cancer and thus might incorporate more noisy and irrelevant information.

\begin{table}[ht]
\centering
\begin{tabular}{ c c c c c c c c }
    \hline
    NDC & NDC-S & NC & NSC & KNN & LDA & SVM & Logistic \\
    \hline
    182(0) & 90(1) & 182(0) & 138(24) & 182(0) & 182(0) & 182(0) & 14(3) \\
    \hline
\end{tabular}
\caption{The means (and standard errors) of the number of selected features on the breast and colon cancer gene expression dataset.}
\label{tab:bc_features}
\end{table}

The number of selected features are reported in Table \ref{tab:bc_features}. For this dataset, logistic regression with $L_1$ regularization only selects 14 features on average (among the three models built for the three folds), which is surprisingly small considering it achieves less than 2\% misclassification rate with less than 8\% of the features. However, comparing NDC-S and NSC, we see that NDC-S also achieve less than 2\% misclassification rate while selecting 90 features on average (around 49\% of the features), whereas NSC achieve more than 12\% misclassification rate while selecting 138 features on average (around 76\% of the features).

\subsection{Leukemia Gene Expression Dataset}
The second dataset consists of 72 samples and 1868 genes. There are two types of samples: 47 samples correspond to acute myeloid leukemia, and 25 samples correspond to acute lymphoblastic leukemia.

\begin{table}[ht]
\centering
\resizebox{\columnwidth}{!}{
\begin{tabular}{ c c c c c c c c }
    \hline
    NDC & NDC-S & NC & NSC & KNN & LDA & SVM & Logistic \\
    \hline
    0.056(0.014) & 0.028(0.014) & 0.056(0.028) & 0.028(0.014) & 0.153(0.077) & 0.167(0.042) & 0.208(0.087) & 0.181(0.091) \\
    \hline
\end{tabular}
}
\caption{The means (and standard errors) of the misclassification rates on the leukemia gene expression dataset.}
\label{tab:leukemia_errors}
\end{table}

The misclassification rates are reported in Table \ref{tab:leukemia_errors}. As we can see, all four centroid-based classification methods (NDC, NDC-S, NC, NSC) achieve misclassification rates that are less than 6\%, whereas the other four classification methods (KNN, LDA, SVM, Logistic) perform significantly worse, with misclassification rates over 15\%.

\begin{table}[ht]
\centering
\begin{tabular}{ c c c c c c c c }
    \hline
    NDC & NDC-S & NC & NSC & KNN & LDA & SVM & Logistic \\
    \hline
    1868(0) & 51(17) & 1868(0) & 1868(0) & 1868(0) & 1868(0) & 1868(0) & 10(3) \\
    \hline
\end{tabular}
\caption{The means (and standard errors) of the number of selected features on the leukemia gene expression dataset.}
\label{tab:leukemia_features}
\end{table}

The number of selected features are reported in Table \ref{tab:leukemia_features}. For this dataset, it is worth noting that despite having equally good performance in terms of misclassification rates, NDC-S only selects 51 features on average (around 3\% of the features), where NSC requires all the features. This is an example where only our NDC-S algorithm could give stellar performance in both classification and feature selection, whereas other competing classifiers could perform well in one aspect at most.

\subsection{Breast Cancer Gene Expression Dataset}
The third dataset consists of 49 samples and 1198 genes. There are two types of samples: 25 samples correspond to breast tumors that are estrogen-receptor-positive, and 24 samples correspond to breast tumors that are estrogen-receptor-negative.

\begin{table}[ht]
\centering
\resizebox{\columnwidth}{!}{
\begin{tabular}{ c c c c c c c c }
    \hline
    NDC & NDC-S & NC & NSC & KNN & LDA & SVM & Logistic \\
    \hline
    0.183(0.032) & 0.145(0.043) & 0.206(0.057) & 0.164(0.043) & 0.186(0.064) & 0.224(0.019) & 0.384(0.104) & 0.163(0.019) \\
    \hline
\end{tabular}
}
\caption{The means (and standard errors) of the misclassification rates on the breast cancer gene expression dataset.}
\label{tab:breast_errors}
\end{table}

The misclassification rates are reported in Table \ref{tab:breast_errors}. As we can see, our NDC-S algorithm again achieves the smallest misclassification rate on this dataset, although the difference between the misclassification rates of most of the classifiers is not that significant after taking the standard errors into consideration.

\begin{table}[ht]
\centering
\begin{tabular}{ c c c c c c c c }
    \hline
    NDC & NDC-S & NC & NSC & KNN & LDA & SVM & Logistic \\
    \hline
    1198(0) & 15(3) & 1198(0) & 445(226) & 1198(0) & 1198(0) & 1198(0) & 12(5) \\
    \hline
\end{tabular}
\caption{The means (and standard errors) of the number of selected features on the breast cancer gene expression dataset.}
\label{tab:breast_features1}
\end{table}

The number of selected features are reported in Table \ref{tab:breast_features1}. For this dataset, we notice that both NDC-S and Logistic are surprisingly efficient at identifying relevant features, selecting 15 and 12 features on average (around 1\% of the features), respectively. In contrast, NSC selects 445 features on average (around 37\% of the features). This shows that our NDC-S algorithm is able to achieve the smallest misclassification rate with as few as 15 features (on average) out of 1198 features.

\begin{table}[ht]
\centering
\begin{tabular}{c c c}
    \hline
    Name & Normalized Frequency & Description \\
    \hline 
    X80062\_at & 0.96 & SA mRNA \\
    29610\_s\_at & 0.80 & GYPE Glycophorin E \\
    X57129\_at & 0.66 & HISTONE H1D\\
    X02958\_at & 0.60 & Interferon alpha gene IFN-alpha 6 \\ 
    X17025\_at & 0.57 & Homolog of yeast IPP isomerase \\
    \hline
\end{tabular}
\caption{The top five most frequently selected genes in the breast cancer gene expression dataset.}
\label{tab:breast_features2}
\end{table}

Since the selected features are genes that might be biologically related to breast cancer, we decide to run the experiment 100 times and compute the normalized frequency of genes that get selected by our algorithm. The top five most frequently selected genes, their normalized frequencies, and their descriptions are listed in Table \ref{tab:breast_features2}. Noticeably, the first two genes, named ``X80062\_at'' and ``29610\_s\_at'', get selected by our algorithm 96\% and 80\% of the time, respectively. This suggests that the biological relationship between these two genes and breast cancer might be worthy of further investigation.

\section{Discussion}\label{sec:discussion}
In this paper, we have developed a new classification method based on nearest centroid, and it is called the nearest disjoint centroid classifier. The two main differences between our nearest disjoint centroid classifier and the nearest centroid classifier is: (1) the centroids are defined based on disjoint subsets of features instead of all the features, and (2) the distance is induced by the dimensionality-normalized norm instead of the Euclidean norm. We have presented and proved a few theoretical results regarding our method. In addition, we have proposed a simple algorithm based on adapted $k$-means clustering that can find the disjoint subsets of features used in our method, and extended the algorithm to perform feature selection by making a few small adjustments. We have evaluated and compared the performance of our method to other classifiers on both simulated data and real-world gene expression datasets. The results have demonstrated that in many situations, our nearest disjoint centroid classifier is able to outperform other competing classifiers by having smaller misclassification rates and/or using fewer features.

In the future, we plan to explore different ways of utilizing the disjoint subsets of features and the associated centroids obtained by our method. In this paper we focused on one simple and straightforward way to perform classification: classify a new data point to the class with the nearest disjoint centroid. However, there are many other methods that could be adapted to using disjoint subsets of features instead of all the features. For example, we could fit a (multinomial) logistic regression model based on the distances from every data point to the $k$ disjoint centroids. We could also define distances from a test data point to a training data point based on the training data point's class and the associated subset of features. Therefore, it is also possible to develop a version of the $k$-nearest neighbors algorithm with disjoint subsets of features.

Another interesting direction to pursue is to consider different ways to obtain the $k$ subsets of features associated with the $k$ classes. In this paper we used an adapted version of the $k$-means clustering algorithm to find those $k$ subsets of features, which is simple but also restrictive: the $k$ subsets of features must be disjoint. In general, our method could still work even if there is intersection between the $k$ subsets of features. As a result, instead of performing $k$-way clustering on the features, we could consider performing two-way clustering on the features $k$ times, each time obtaining one group of features for one class. In the end, we would obtain $k$ groups of features, and they are not required to be disjoint. In addition, they are also not required to cover all the features, and the features that are not included in any of the $k$ groups are not used in prediction. This means that it could also perform feature selection, although controlling the number of selected features would require additional work.

\bibliography{ref}

\end{document}